\documentclass{new_tlp}
\usepackage{mathptmx}
\usepackage{xcolor}
\usepackage{xspace}
\usepackage{url}
\usepackage{amssymb}
\usepackage[normalem]{ulem}

\newcommand\bcmdtab{\noindent\bgroup\tabcolsep=0pt%
  \begin{tabular}{@{}p{10pc}@{}p{20pc}@{}}}
\newcommand\ecmdtab{\end{tabular}\egroup}

\newcommand{\estep}{\mathit{endstep}}
\newcommand{\affected}{\mathit{affected}}
\newcommand{\move}{\mathit{move}}

\newcommand{\onNode}{\mathit{onNode}}
\newcommand{\step}{\mathit{step}}
\newcommand{\edge}{\mathit{edge}}
\newcommand{\edgeP}{\mathit{edgeP}}
\newcommand{\node}{\mathit{node}}
\newcommand{\inInduced}{\mathit{inInduced}}
\newcommand{\inClique}{\mathit{inClique}}

\newcommand{\QASP}{ASP(Q)\xspace}

  \title[Beyond NP: Quantifying over Answer Sets]
        {Beyond NP: Quantifying over Answer Sets}

  \author[G. Amendola, F. Ricca,  M. Truszczynski]
         {GIOVANNI AMENDOLA$^1$  \and FRANCESCO RICCA$^1$ \and MIREK TRUSZCZYNSKI$^2$\\
         $^1$University of Calabria, Rende, Italy\\
         \email{\{amendola,ricca\}@mat.unical.it}\\
         $^2$University of Kentucky, KY, USA\\
         \email{mirek@cs.uky.edu}}

\jdate{March 2003}
\pubyear{2003}
\pagerange{\pageref{firstpage}--\pageref{lastpage}}
\doi{S1471068401001193}

\newtheorem{theorem}{Theorem}
\newtheorem{proposition}{Proposition}

\newtheorem{example}{Example}

\begin{document}

\label{firstpage}

\maketitle

\begin{abstract}
Answer Set Programming (ASP) is a logic programming paradigm featuring a 
purely declarative language with comparatively high modeling capabilities. 
Indeed, ASP can model problems in NP in a compact and elegant way. However, 
modeling problems beyond NP with ASP is known to be complicated, on the one 
hand, and limited to problems in $\Sigma^P_2$ on the other.
Inspired by the way Quantified Boolean Formulas extend SAT formulas to model 
problems beyond NP, we propose an extension of ASP that introduces quantifiers
over \emph{stable models} of programs. We name the new language ASP with Quantifiers
(\QASP).
In the paper we identify computational properties of \QASP; we highlight its 
modeling capabilities by reporting natural encodings of several complex 
problems with applications in artificial intelligence and number theory; and 
we compare \QASP with related languages. Arguably, \QASP allows one to model 
problems in the Polynomial Hierarchy in a direct way, providing
an elegant expansion of ASP beyond the class NP.
Under consideration for acceptance in TPLP.
\end{abstract}

  \begin{keywords}
    ASP, Quantified Logics, Polynomial Hierarchy
  \end{keywords}

%\tableofcontents

%
\section{Introduction}
%ASP~\cite{DBLP:journals/cacm/BrewkaET11}.
Answer Set Programming (ASP) 
\cite{DBLP:journals/cacm/BrewkaET11} %MarekT99,Niemela99,
is a logic programming paradigm for modeling and solving search and 
optimization problems. It is supported by a purely declarative formalism of 
logic programs with the semantics of stable models~\cite{DBLP:journals/ngc/GelfondL91} (also known as \emph{answer sets}~\cite{Lifschitz02}), and by several systems able to compute them
\cite{DBLP:conf/ijcai/GebserLMPRS18}. 
ASP was primarily aimed at problems whose decision versions are in the class 
NP. Indeed, ASP can model problems in NP in a compact and elegant way by means
of an intuitive and easy to follow methodology known as generate-define-test 
\cite{Lifschitz02} (also known as guess and check~\cite{EiterGandC2000}). 
Furthermore, implementations such as \emph{clasp}~\cite{DBLP:conf/lpnmr/GebserKK0S15}, and \emph{wasp}~\cite{DBLP:conf/lpnmr/AlvianoDLR15} % and \emph{DLV}~\cite{DBLP:journals/tocl/LeonePFEGPS06}
have been shown to be effective in solving problems of 
practical interest on industrial-grade instances \cite{DBLP:journals/aim/ErdemGL16}.
%\cite{DBLP:journals/tplp/DodaroGLMRS16,DBLP:journals/tplp/GebserOSR18,DBLP:journals/aim/ErdemGL16}.

Modeling problems beyond the class NP with ASP is possible to some extent.
Namely, when disjunctions are allowed in the heads of rules, every decision 
problem in the class $\Sigma_2^P$ can be modeled in a uniform way by a finite
program \cite{DBLP:journals/csur/DantsinEGV01}. However, modeling problems beyond NP with
ASP is complicated and the generate-define-test approach is no longer sufficient in general. 
Additional techniques such as \emph{saturation} \cite{DBLP:journals/amai/EiterG95} are needed but they are difficult to use, and may introduce constraints that have no direct relation to constraints of the problem being modeled.
As stated explicitly in~\cite{DBLP:journals/tplp/GebserKS11}
``unlike the ease of common ASP modeling, [...] these techniques are rather involved and hardly usable by ASP laymen.''

The primary goal of our work is to address the shortcomings of ASP in modeling
problems beyond NP. Building on the way Quantified Boolean formulas (QBFs)
extend SAT formulas to model problems from PSPACE, we propose a generalization 
of ASP that introduces quantifiers over stable models of programs. We name the 
new language \textit{ASP with Quantifiers} (\QASP) and refer to programs in that language as
\emph{quantified programs}. 

In the paper we formally introduce the language \QASP and its semantics. We
identify computational properties of \QASP. In particular, we show that every
problem in the Polynomial Hierarchy can be uniformly modeled by a quantified program. Moreover,
we show that no loss of expressivity results if we restrict programs defining 
quantifiers to be normal. An important consequence of that observation is that
when using \QASP to model problems, one can resort to the generate-define-test 
approach to specify these ``quantifying'' programs. This typically simplifies
modeling and verifying correctness. We illustrate these claims by presenting 
natural encodings of several complex problems with applications in artificial 
intelligence and mathematics. 

In the last part of the paper, we compare \QASP with alternative approaches for modeling problems beyond NP.
Earlier efforts in this direction include: the \emph{stable-unstable} formalism~\cite{DBLP:journals/tplp/BogaertsJT16},
various program transformations~\cite{DBLP:journals/tplp/EiterP06,DBLP:conf/lpnmr/Redl17a,DBLP:conf/birthday/FaberW11}, applications of meta-programming~\cite{DBLP:conf/lpnmr/Redl17a,DBLP:journals/tplp/GebserKS11} and more.%
\footnote{For example, weak constraints allow to model decision problems that are $\Delta_3^P$-complete~\cite{DBLP:journals/tkde/BuccafurriLR00}.}
In particular, we deepen the comparison with disjunctive programs and the stable-unstable formalism,
indicating key differences and their implications by means of additional modeling examples. 
We also extensively compare \QASP with the language of QBFs, which served as a direct inspiration for our work. 
%
%Eventually we draw the conclusion and mention some ongoing and future work.
%The bottom line 
A single sentence summary of our work is: \QASP allows one to model 
problems in the Polynomial Hierarchy in a direct way, providing
an elegant expansion of ASP beyond the class NP.
%
% Unlike \QASP, disjunctive ASP
%is limited to problems in the class $\Sigma_2^P$ and the stable-unstable 
%formalism also is studied primarily in the context of problems in that 
%class.\footnote{However, we note that a generalization of the stable-unstable 
%formalism to higher levels of the poynomial hierarchy is evident and it was 
%mentioned by Bogaerts et al. \citeyear{DBLP:journals/tplp/BogaertsJT16}.}

\section{Formal Framework}

We start by recalling syntax and semantics of \textit{Answer Set Programming} 
(ASP). We then introduce syntax and semantics of \textit{ASP with Quantifiers} 
(\QASP).
%In both cases, we concentrate on logic programs over a propositional signature, since extensions to non-ground logic programs are straightforward by using grounders.

\subsection{Answer Set Programming}

Let $\mathcal{R}$ be a set of predicates, $C$ a set of constants, and $V$ a set of variables. A \textit{term} is a constant or a variable. An atom $a$ of arity $n\in\mathbb{N}$ is of the form $p(t_{1},..., t_{n})$, where $p$ is a predicate from $\mathcal{R}$ and $t_{1},..., t_{n}$ are terms. 
A \textit{disjunctive rule} $r$ is of the form
\begin{equation}
 a_{1}\vee\ldots\vee a_{l} \leftarrow b_{1},\ldots,b_{m}, \ not \
 c_{1},\ldots,\ not \ c_{n},
\label{eq:rule}
\end{equation}
\noindent where all $a_i$, $b_j$, and $c_k$ are atoms; $l,m,n\geq 0$ and $l+m+n >0$; $not$ represents
\textit{negation-as-failure}, also known as \textit{default negation}. 
The set $H(r)=\lbrace
a_{1},...,a_{l} \rbrace$ is the \textit{head} of $r$; the sets $B^{+}(r)=
\lbrace b_{1},...,b_{m} \rbrace$ and $B^{-}(r)=\lbrace c_{1},\ldots,c_{n} 
\rbrace$ are
the sets of the \textit{positive body} and the \textit{negative body} atoms of $r$, respectively. 
%; and $B(r)=B^{+}(r)\cup B^{-}(r)$ is the \textit{body} of $r$.
A rule $r$ is \textit{safe} if each of its variables occurs in some positive 
body atom. We restrict attention to programs built of safe rules only.
A rule $r$ is a \textit{fact}, if $B^{+}(r)\cup B^{-}(r)=\emptyset$ (we then omit
$\leftarrow$ from the notation);
a \textit{constraint}, if $H(r)=\emptyset$;
\textit{normal}, if $| H(r)| \leq 1$;
and \textit{positive}, if $B^{-}(r)=\emptyset$.
A \textit{(disjunctive logic) program} $P$ is a
finite set of disjunctive rules. $P$ is called \textit{normal}
[resp.\ \textit{positive}] if each $r\in P$ is normal [resp.\ positive]. 
\textcolor{black}{We define $At(P)=\bigcup_{r\in P} At(r)$,
that $At(P)$ is the set of all atoms occurring in the program $P$.}
A program $P$ is \textit{stratified} if there is a level mapping  $\|.\|_s$ of $P$ such that for every rule $r$ of $P$:
$(i)$ For any predicate $p$ occurring in $B^{+}(r)$, and for any $p'$ occurring in $H(r)$,  $\| p \|_s \leq \| p' \|_s$, and 
$(ii)$ For any predicate $p$ occurring in $B^{-}(r)$, and for any $p'$ occurring in $H(r)$,  $\| p \|_s < \| p' \|_s$.
%Roughly, negation and recursion. 

The {\em Herbrand universe} of $P$, denoted by $U_{P}$, is the set of all 
constants appearing in $P$, except that when no constants appear in $P$, we
take $U_{P}=\{a\}$, where $a$ is an arbitrary constant.
The {\em Herbrand base} of $P$, denoted as $B_{P}$, is the set of all ground atoms that can be obtained
from the predicate symbols appearing in $P$ and the constants of $U_{P}$.
Given a rule $r$ occurring in a program $P$, a {\em ground instance} of  $r$ is a rule obtained from $r$ by replacing every variable $X$ in $r$ by $\sigma (X)$, where $\sigma$  is a substitution mapping the variables occurring in $r$ to
constants in $U_{P}$. 
The \textit{ground instantiation} of $P$, denoted by $ground(P)$, is the set of all the ground instances of the rules occurring in $P$. 
Any set $I\subseteq B_P$ is an \textit{interpretation}; it is a \textit{model} 
of a program $P$ (denoted $I\models P$) if for each rule $r\in ground(P)$, 
we have $I\cap H(r)\neq \emptyset$ whenever $B^{+}(r)\subseteq I$ and 
$B^{-}(r)\cap I=\emptyset$ (in such case, $I$ is a model of $r$, denoted $I
\models r$). A model $M$ of $P$ is \textit{minimal} if no model $M'\subset M$
of $P$ exists.
We denote by $MM(P)$ the set of all minimal models of $P$. 
\textcolor{black}{
For a program $P$ without constraints we write $P^I$ for the well-known 
\textit{Gelfond-Lifschitz reduct}~\cite{DBLP:journals/ngc/GelfondL91}
with respect to interpretation $I$, that is, the set of rules $H(r) \leftarrow
B^{+}(r)$, obtained from rules $r\in ground(P)$ such that $B^-(r) \cap I= \emptyset$. 
An answer set (or stable model) of a program $P$ without constraints is an interpretation $I$ such that $I\in MM(P^I)$.
For the general case, we write $P_{\leftarrow}$ for the set of constraints of a disjunctive logic program $P$.
We denote by $AS(P)$ the set of all {\em answer sets (or stable models)} of such programs $P$, that is, the set of all answer sets of $P\setminus P_{\leftarrow}$ that are models for $P_{\leftarrow}$.
}
We say that a program $P$ is \textit{coherent}, if it has at least one answer
set (that is, $AS(P) \neq \emptyset$), otherwise, $P$ is \textit{incoherent}.

\subsection{Answer Set Programming with Quantifiers}

%Let $P$ be a logic program, and $I$ be an interpretation over $B_P$.
%We denote by $fix_P(I)$ the set of facts and constraints 
%$\{ a \mid a\in I \} \cup \{ \leftarrow a \mid a\in B_P\setminus I \}$.
%%$\{ a \mid a\in I \}\cup\{\leftarrow a \mid a\in At(ground(P))\setminus I \}$.
%
An \textit{ASP with Quantifiers} (\QASP) program $\Pi$ is an expression of the form:
\textcolor{black}{
\begin{equation}
\Box_1 P_1\ \Box_2 P_2\ \cdots\ \Box_n P_n :  C ,
\label{eq:qasp}
\end{equation}
}
where, for each $i=1,\ldots,n$, $\Box_i \in \{ \exists^{st}, \forall^{st}\}$, $P_i$ is an ASP program, 
and $C$ is a stratified normal ASP program.%
\footnote{This condition is sufficient to model compactly constraints by exploiting the modeling advantages of inductive definitions.  $C$ is contemplated in the definition of \QASP just because it makes more natural the modeling of problems.}
Symbols $\exists^{st}$ and $\forall^{st}$ are named \textit{existential} and \textit{universal answer set quantifiers}, respectively.
\textcolor{black}{An 
\QASP program $\Pi$ of the form (\ref{eq:qasp}) is \emph{existential} 
(\emph{universal}, respectively) if $\Box_1 =\exists^{st}$ ($=\forall^{st}$, 
respectively).}
If for each $i=1,\ldots,n$ the ASP program $P_i$ is normal, then $\Pi$ is 
called a \textit{normal} \QASP program.
 %s with constraints. 
%A \textit{quantified interpretation} $I$ is a sequence $(I_1,\ldots,I_n)$ of interpretations.
%A quantified interpretation $M=(M_1,\ldots,M_n)$ is a model of $\Pi$ if 
Given a logic program $P$ and an intepretation $I$ over $B_{P}$, and an \QASP 
program $\Pi$ the form~(\ref{eq:qasp}), we denote by $fix_P(I)$ the set of 
facts and constraints          
$\{ a \mid a\in I \} \cup \{ \leftarrow a \mid a\in B_P\setminus I \}$,
%$\{ a \mid a\in I \}\cup\{\leftarrow a \mid a\in At(ground(P))\setminus I \}$,
and by $\Pi_{P,I}$ the \QASP program of the form~(\ref{eq:qasp}), where $P_1$ is 
replaced by $P_1\cup fix_P(I)$, that is, $\Pi_{P,I}$ $=$ $\Box_1 (P_1\cup fix_P(I)) \cdots \Box_n P_n :  C$.
We now define \textit{coherence} of \QASP programs by induction on the number of 
quantifiers in the program.
\begin{itemize}
\item $\exists^{st} P:C$ is coherent, if there exists $M\in AS(P)$ such that 
$C\cup fix_P(M)$ is coherent;
\item $\forall^{st} P:C$ is coherent, if for every $M\in AS(P)$, 
$C\cup fix_P(M)$ is coherent;
\item $\exists^{st} P\ \Pi$ is coherent, if there exists $M\in AS(P)$ such 
that $\Pi_{P,M}$ is coherent;
\item $\forall^{st} P\ \Pi$ is coherent, if for every $M\in AS(P)$, $\Pi_{P,M}$
is coherent.
\end{itemize}
For instance, an \QASP program 
$\Pi=\exists^{st} P_1 \forall^{st} P_2 \cdots \exists^{st} P_{n-1} \forall^{st} P_n: C$ 
is coherent if there exists an answer set $M_1$ of $P_1'$ 
such that for each answer set $M_2$ of $P_2'$ 
there is an answer set $M_3$ of $P_3', \ldots,$ 
there is an answer set $M_{n-1}$ of $P_{n-1}'$ 
such that for each answer set $M_n$ of $P_n'$, 
there is an answer set of $C\cup fix_{P_n'}(M_n)$, 
where $P_1'=P_1$, and $P_i'=P_i\cup fix_{P_{i-1}'}(M_{i-1})$, if $i\geq2$. 

For an \QASP program $\Pi$ of the form~(\ref{eq:qasp}) such that $\Box_1=\exists^{st}$, we say that $M \in AS(P_1)$  is a \textit{quantified answer set} of $\Pi$, whenever $(\Box_2 P_2 \cdots \Box_n P_n :  C)_{P_1,M}$ is coherent, in case of $n>1$, and whenever $C\cup fix_{P_1}(M)$ is coherent, in case of $n=1$.
We denote by \textcolor{black}{$QAS(\Pi)$} the set of all quantified answer sets of $\Pi$.
Finally, note that the definition of quantified answer set can be naturally extended to programs with strong negation,  choice rules%
%~\cite{DBLP:journals/ai/SimonsNS02}
, aggregates
%~\cite{DBLP:conf/jelia/FaberLP04} 
and other extensions~\cite{DBLP:journals/aim/GebserS16}. 
Thus, in the examples we resort also to these extensions that are part of the ASPCore standard input language~\cite{DBLP:conf/ijcai/GebserLMPRS18}. %~\cite{aspcore}. 

\begin{example}
Consider the \QASP program $\Pi =\exists^{st} P_1 \forall^{st} P_2:C$, where
$P_1=\{a(1) \vee a(2)\}$, $P_2=\{b(1) \vee b(2) \leftarrow a(1); \ b(2)
\leftarrow a(2)\}$, and $C=\{\leftarrow b(1),\ not \ b(2)\}$. 
The program $P_1$ has two answer sets $\{a(1)\}$ and $\{a(2)\}$.
Hence, to establish the coherence of $\Pi$, we have to check if at least one 
of $\{a(1)\}$ and  $\{a(2)\}$ is a quantified answer set of $\Pi$. 
Considering $\{a(1)\}$, we have $fix_{P_1}(\{a(1)\})= \{a(1); 
\leftarrow a(2)\}$. Under the notation used above, $P_2'= P_2\cup fix_{P_1}
(\{a(1)\})$. Thus, $AS(P_2\cup fix_{P_1}(\{a(1)\}))=\{ \{a(1),b(1)\},
\{a(1),b(2)\}\}$. For $M=\{a(1),b(1)\}$ we have $fix_{P_2'}(M)=
\{a(1);$ $b(1);\;\leftarrow a(2);\;\leftarrow b(2)\}$, and it is clear that
the program $C\cup fix_{P_2'}(M)$ is not coherent. Therefore, 
$\{a(1)\}$ is not a quantified answer set of $\Pi$. On the other hand, a 
similar analysis for the other answer set of $P_1$, $\{a(2)\}$, shows that
it is a quantified answer set of $\Pi$.
\end{example}

\QASP is a straightforward generalization of ASP in a sense made formal in 
the following theorem.
%\footnote{The proof of this
%result and of some  other theorems are given in the appendix available
%as supplemental materials published with the paper.}

\begin{theorem}
Let $P$ be an ASP program, and let $\Pi$ be the \QASP program of the form~(\ref{eq:qasp}), where $n=1$,  $\Box_1=\exists^{st}$, $P_1=P$, and $C=\emptyset$.
Then, $AS(P)=QAS(\Pi)$.
\end{theorem}
\begin{proof}
By definition, $M$ is a quantified answer set of $\Pi$ if and only if $M$ is an 
answer set of $P$ and $\emptyset\cup fix_P(M)=fix_P(M)$ is coherent. The latter
condition is trivially true as $M$ is an answer set of $fix_P(M)$.
\end{proof}

%\begin{proof}
%By definition, $M$ is a quantified answer set of $\Pi$ if and only if $M$ is an 
%answer set of $P$ and $\emptyset\cup fix_P(M)=fix_P(M)$ is coherent. The latter
%condition is trivially true as $M$ is an answer set of $fix_P(M)$.
%\end{proof}

\section{Complexity issues}\label{sec:complexity}
We now study the computational properties of the \QASP language.
As it is customary in the literature we focus on the ground case, that is we assume that no variable occurs in programs. 

Because it is possible to alternate universal and existential answer set 
quantifiers, it is clear that \QASP can model probelms beyond NP.
In particular, each problem in PSPACE %the polynomial hierarchy 
can be modeled by using an \QASP program. Formally, we define the {\sc Coherence} 
problem as follows: Given an \QASP program $\Pi$ as input, decide whether $\Pi$ is coherent.

\begin{theorem}\label{th:pspace}
The  {\sc Coherence} problem is PSPACE-complete, even under the restriction to
normal \QASP programs.
\end{theorem}
\begin{proof} (Membership)
It is well known that answer sets of a disjunctive logic program can be 
enumerated in polynomial space in the size of the program. Let us assume  
that $p$ is a polynomial providing that bound. We prove that the 
coherence of an $\QASP$ program $\Pi$ of the form (\ref{eq:qasp}) can be 
decided in space $O(n\times p(s(\Pi)))$, where $s(\Pi)$ is the size of 
$\Pi$, and $n$ is the number of quantifiers in $\Pi$.

\textcolor{black}{To this end, we consider the following recursive algorithm. It consists of
enumerating all answer sets of $P_1$. If $n=1$, we have $\Pi=\Box P_1:C$. 
To decide coherence, for each enumerated answer set $M$ of $P_1$, we decide 
whether $C\cup fix_{P_1}(M)$ is coherent. Depending on whether $\Box=
\exists^{st}$ or $\forall^{st}$, if for some (every) answer set $M$ of 
${P_1}$, $C\cup fix_{P_1}(M)$ is coherent, we return that $\Pi$ is coherent. 
Otherwise, we return that $\Pi$ is not coherent. For $n\geq 2$, 
for each enumerated answer set $M$ of ${P_1}$, we recursively check whether 
$\Pi'=(\Box_2 P_2\ldots\Box_n P_n:C)_{P_1,M}$ is coherent, and decide about
coherence of $\Pi$ similarly as in the case $n=1$, depending on
the outermost quantifier.} 

\textcolor{black}{
By the comment above, we can enumerate all answer sets $M$ of $P_1$ in space 
$O(p(s(\Pi)))$ (indeed, $s(P_1)= O(s(\Pi))$). Moreover, if $n=1$, testing 
coherence of $C\cup fix_{P_1}(M)$ can be accomplished in time and so, also 
in space $O(s(C\cup fix(M)))=O(s(\Pi))$. Thus, if $n=1$, the algorithm requires
$O(p(s(\Pi)))$ space, establishing the base case of the induction. If 
$n\geq 2$, we need $O(p(s(\Pi)))$ space for enumerating answer sets and,
using the induction hypothesis, $O((n-1) \times p(s(\Pi)))$ space for each 
recursive call. Thus, the total space requirement is $O(n \times p(s(\Pi)))$,
completing the inductive step.}

\textcolor{black}{
We now observe that $n=O(s(\Pi))$, which shows that the algorithm we described
runs in space $O(s(\Pi)\times p(s(\Pi)))$. This implies the assertion.} 
\smallskip
\noindent

(Hardness) We give a reduction from the problem of deciding the
validity of a QBF formula $\Phi=Q_1 x_1 \ldots Q_n x_n \varphi$,
where for every $i=1,\ldots, n$, $Q_i\in \{\exists,\forall\}$ and $x_i$
is a propositional variable, and where $\varphi$ is a propositional formula
over $\{x_1,\ldots,x_n\}$. 
The problem is PSPACE-complete even when $\varphi$ is in 3-CNF. Thus, let us 
assume that $\varphi = C_1 \wedge \ldots \wedge C_m$, where $C_j = l_j^1 \vee 
l_j^2\vee l_j^3$ and $l_j^1, l_j^2,l_j^3 \in \{x_i,\neg x_i\mid i=1,\ldots,n\}$,
for each $j=1,\ldots,m$. We construct an \QASP program $\Pi$ as follows. For 
each $i=1,\ldots,n$, we define $P_i = \{ x_i \leftarrow not \ nx_i; \ nx_i 
\leftarrow not \ x_i \}$ and $\Box_i = Q_i^{st}$. We also define $C= \{ ok_j 
\leftarrow \sigma(l_j^h) \mid j=1,\ldots,m \mbox{ and } h=1,2,3 \} \cup 
\{\;\leftarrow not\ ok_i\mid i=1,\ldots, m\}$, 
where $\sigma(l) = x_i$ if $l=x_i$, and
$\sigma(l) = nx_i$ if $l = \neg x_i$. 
It is easy to see that $\Pi$ is coherent iff $\Phi$ is valid. Moreover,
as each program $P_i$ is normal, $\Pi$ is a normal \QASP program.
\end{proof}

As for QBFs, there is a direct correspondence between the 
number of alternating quantifiers and the level of the Polynomial Hierarchy (PH)
for which we have competeness of the coherence problem.

%\begin{theorem}\label{th:norm}
%Let $\Pi$ be an \QASP program of the form (\ref{eq:qasp}) having $k$ quantifier 
%alternations.
%The  {\sc Coherence} problem  is:
%$(i)$ $\Sigma^P_k$-complete for normal \QASP program whenever $\Box_1=\exists^{st}$; and
%$(ii)$  $\Pi^P_k$-complete for normal \QASP program whenever $\Box_1=\forall^{st}$.
%
%%$(i)$ $\Sigma^P_k$-complete (resp. $\Pi^P_k$-complete) for normal \QASP program whenever $\Box_1=\exists^{st}$ (resp. $\Box_1=\forall^{st}$).
%%$(ii)$ $\Sigma^P_{2k}$-complete (resp. $\Pi^P_{2k}$-complete) for disjunctive \QASP program whenever $\Box_1=\exists^{st}$ (resp. $\Box_1=\forall^{st}$).
%\end{theorem}

\begin{theorem}\label{th:norm}
The {\sc Coherence} problem is $(i)$ $\Sigma^P_n$-complete for normal 
\textcolor{black}{existential} \QASP programs with $n$ quantifiers in the 
prefix; and $(ii)$ $\Pi^P_n$-complete for normal \textcolor{black}{universal}
\QASP programs with $n$ quantifiers in the prefix.
\end{theorem}
\begin{proof}
(Membership) We proceed by induction on $n$. We start with $n=1$. If $\Pi = 
\exists^{st} {P_1}:C$ then deciding coherence amounts to checking whether there is
an answer set $I$ of ${P_1}$ such that $fix_{P_1}(I) \cup C$ is coherent. This problem
is in NP ($=\Sigma_1^P$) because one can check coherence of a normal 
stratified program with constraints in polynomial time~\cite{DBLP:journals/csur/DantsinEGV01}. 
If $\Pi = \forall^{st} {P_1}:C$ then deciding coherence amounts to checking whether 
there is no answer set $I$ of ${P_1}$ such that $fix_{P_1}(I) \cup C$ is not coherent.
This problem is in co-NP ($=\Pi_1^P$) because its complement, the problem to 
decide whether there is an answer set $I$ of ${P_1}$ such that $fix_{P_1}(I) \cup C$ 
is not coherent, is in NP (indeed, one can check coherence of a 
normal stratified program with constraints in polynomial time).

Next, let us assume that $n\geq 2$. Further, let $\Pi$ be a normal \QASP program
of the form~(\ref{eq:qasp}). If $\Box_1=\exists^{st}$, then to decide coherence
of $\Pi$ we have to decide whether there is an interpretation $I$ such that $I$
is an answer set of $P_1$ and $\Pi_{P_1,I}$ is coherent. Checking that $I$ is 
an answer set of $P$ is a polynomial-time task (we recall that $P_1$ is normal).
Checking that $\Pi_{P_1,I}$ is coherent can be accomplished with a call to an 
oracle for a problem in $\Sigma_{n-1}^P$ or in $\Pi_{n-1}^P$ depending on 
whether $\Box_2$ in $\Pi$ is $\exists^{st}$ or $\forall^{st}$. Indeed, by the induction 
hypothesis, the problem of deciding coherence for normal $\QASP$ programs with 
$n-1$ quantifiers and with the outermost quantifier fixed to $\exists^{st}$ 
($\forall^{st}$, respectively) is in $\Sigma_{n-1}^P$ ($\Pi_{n-1}^P$, 
respectively).

If $\Box_1=\forall^{st}$, to decide coherence of $\Pi$ we have to 
decide that for every answer set of $P_1$, $\Pi_{P_1,I}$ is coherent. The 
complement to this problem consists of deciding whether there is an 
an answer set $I$ of $P_1$ such that $\Pi_{P_1,I}$
is not coherent. By a similar argument as above, this problem is in 
$\Sigma_{n}^P$ (observe that an oracle deciding whether an $\QASP$ program
is coherent, can be used to decide whether an $\QASP$ program is not coherent).
It follows that deciding coherence for programs with $n$ quantifiers in the 
prefix and with $\forall^{st}$ as the outermost quantifier is in $\Pi_{n}^P$.

\smallskip
\noindent
(Hardness) Let us consider a QBF $\Phi =Q_1 X_1 \ldots Q_n X_n \varphi$, where 
$X_1,\ldots,X_n$ are disjoint sets of propositional variables, each $Q_i=
\exists$ or $\forall$, the quantifiers alternate, and $\varphi$ is a 3-CNF or 
3-DNF formula over the variables in $X_1\cup\ldots\cup X_n$. We encode $\Phi$
as an $\QASP$ program $\Pi_\Phi$ of the form (\ref{eq:qasp}) as follows. For 
every $i=1,\ldots, n$, we set \textcolor{black}{$\Box_i=Q_i^{st}$} and $P_i=\{x \leftarrow not \ nx
\mid x\in X_i\}\cup \{nx \leftarrow not \ x \mid x\in X\}$ (similarly as 
in the previous proof). If $\varphi$ is a 3-CNF formula, we define a normal 
stratified program with constraints $C$ as in the previous proof. So, assume 
$\varphi$ is a 3-DNF formula, say $\varphi = D_1 \vee \ldots \vee D_m$, where
$D_j = l_j^1 \wedge l_j^2\wedge l_j^3$ and $l_j^1, l_j^2,l_j^3 \in X_1\cup 
\ldots\cup X_n$, for each $j=1,\ldots,m$. In this case, we set $C= \{ ok_j 
\leftarrow \sigma(l_j^1), \sigma(l_j^2), \sigma(l_j^3) \mid j=1,\ldots,m\} 
\cup \{ \leftarrow not\ ok_1,\ldots, not\ ok_m\}$, where $\sigma(l) = x$
if $l=x$, and $\sigma(l) = nx$ if $l= \neg x$.   

It is easy to see that in both cases $\Phi$ is valid iff $\Pi_\Phi$ is coherent.
Moreover, both encodings can be obtained by a polynomial-time procedure. Now, 
according to well-known complexity results~\cite{DBLP:journals/tcs/Stockmeyer76} 
the problem to decide validity for  QBFs such that
(1) $Q_1=\exists$, $\varphi$ is in 3-DNF, and $n$ is even; 
(2) $Q_1=\exists$, $\varphi$ is in 3-CNF, and $n$ is odd; 
(3) $Q_1=\forall$, $\varphi$ is in 3-CNF, and $n$ is even; 
(4) $Q_1=\forall$, $\varphi$ is in 3-DNF and $n$ is odd 
is $\Sigma_n^P$-complete for the cases (1) and (2), and 
$\Pi_n^P$-complete for the cases (3) and (4). Thus, the hardness follows.
\end{proof}

We note that, for classes of disjunctive programs that can be translated in 
polynomial time to normal ones, such as Head-Cycle Free 
(HCF)~\cite{DBLP:journals/amai/Ben-EliyahuD96}, the correspondence between 
quantifier alternations and the level of the Polynomial Hierarchy is preserved.

\textcolor{black}{We also note that the theorem concerns, in each of the 
two cases, the corresponding class of \emph{all} \QASP programs with $n$ 
quantifiers. In particular, the membership part is proved for that class.
The proof of hardness explicitly usues special programs in that class, the 
ones in which quantifiers \emph{alternate}.}

\section{Modeling in \QASP}\label{sec:modeling}
In this section, we focus on the modeling capabilities of our language.
Thus, we study some well-known problems that are computationally beyond NP,
and show how to solve them in \QASP.
 
\subsection{Minmax Clique}
\textit{Minmax problems} play a key role in  various  fields of research, 
including game theory, combinatorial optimization and computational 
complexity~\cite{Cao1995}.
A minimax problem can be formulated as $min_{x\in X}  max_{y\in Y}  f(x,y)$,
where  $f(x,  y)$  is  a  function  defined  on  the  product set  of  $X$ 
and  $Y$. Here, we focus on the so-called \textit{Minmax Clique} 
problem~\cite{Ko1995}, but our approach can be easily adapted to model 
other minmax problems.

Let $G=\langle N,E\rangle$ be a graph, $I$ and $J$ two finite sets of indices,
and $(A_{i,j})_{i\in I,j\in J}$ a partition of $N$. We write $J^I$ for the
set of all total functions from $I$ to $J$.  
For every total function $f\colon I\rightarrow J$ we denote by $G_f$ the 
subgraph of $G$ induced by $\bigcup_{i\in I} A_{i,f(i)}$. We define the 
{\sc Minmax Clique} problem as follows: Given a graph $G$, sets of indices 
$I$ and $J$, a partition $(A_{i,j})_{i\in I,j\in J}$
%$\bigcup_{i\in I} A_{i,f(i)}$
(all as above), and 
an integer $k$, decide whether 
\[
%\min_{f\in J^I}\;\max_{Q\subseteq N}\{|Q|: \mbox{$Q$ is a clique of $G_f$}\}\geq k.
\min_{f\in J^I}\;\max\{|Q|: \mbox{$Q$ is a clique of $G_f$}\}\geq k.
\]
%
%
%and let $J^I$ be the set of all total functions from $I$ to $J$.
%Given a graph $G=\langle N,E\rangle$, a partition of its nodes of cardinality $|I\times J|$ is denoted by
%$(A_{i,j})$, where $i$ varies in  $I$ and $j$ in $J$, and a function $f\in J^I$, 
%%
%we define $\mathit{LongestCycle}(f)$ as the length of the longest cycle in $G$ restricted to the subset of nodes given by
%$\bigcup_{i\in I} A_{i,f(i)}$.
%
%The Minmax Clique problem takes as input a graph $G=\langle N,E\rangle$; a partition $(A_{i,j})$ of its nodes, as defined above; and an integer $k$.
%The problem asks whether the minimum of the set $\{ \mathit{LongestCycle}(f) \mid f\in J^I\}$ is greater or equal to $k$.
It is known that this problem is $\Pi_2^p$-complete~\cite{Ko1995}.

Consider the following \QASP program $\Pi = \forall^{st} P_1 \exists^{st} P_2: C$. 
The ASP program $P_1$ is given by:
\begin{center}
$P_1=\left\{
\begin{array}{rcll}
\edge(a,b) && &\forall (a,b)\in E\\
\node(a) && &\forall a\in N\\
v(i,j,a) && & \forall i\in I, \ j\in J, \ a\in A_{i,j} \\
setI(X) &\leftarrow& v(X,\_,\_)\\
setJ(X) &\leftarrow& v(\_,X,\_)\\
1 \{ f(X,Y) :  setJ(Y) \} 1 &\leftarrow& setI(X)
\end{array}\right\}$
\end{center}
Informally, the role of $P_1$ is to specify the input graph, the sets
$I$ and $J$ of indices, a partition $(A_{i,j})$, and the search space of all total
functions from $I$ to $J$.  
Specifically, the first two sets of facts encode the graph by using two 
predicates: a binary one named $\edge$, collecting all edges of the graph; 
and a unary one named $\node$ collecting all nodes of the graph. Then, the 
third set of facts encodes the  partition $(A_{i,j})$ by using a ternary 
predicate $v$. Projections applied to $v$ (rules four and five) define elements 
of the sets $I$ and $J$, respectively. Finally, the last rule defines the 
space of all total functions $f$ from $I$ to $J$.
The ASP program $P_2$ is defined as follows:
\begin{center}
$P_2=\left\{
\begin{array}{rcl}
\inInduced(Z) &\leftarrow& v(X,Y,Z), \ f(X,Y)\\
\edgeP(X,Y) &\leftarrow& \edge(X,Y), \ \inInduced(X), \ \inInduced(Y)\\
\{ \inClique(X)\, : \, \inInduced(X) \} &&\\
&\leftarrow& \inClique(X), \ \inClique(Y), \ not \ \edgeP(X,Y)
\end{array}\right\}$
\end{center}
Its role is to define the subgraph $G_f$ of $G$ determined by a total 
function $f$, and to select a clique in this subgraph. In particular, the 
first rule defines the set of nodes of the subgraph $G_f$ (whenever a node
$Z$ belongs to the set $A_{X,Y}$, and the function $f$ maps $X$ to $Y$, then 
$Z$ is a node of $G_f$). The second rule ensures that whenever there is an 
edge from $X$ to $Y$, and both $X$ and $Y$ are nodes of $G_f$, then the edge 
$(X,Y)$ is an edge of $G_f$ ($G_f$ is the \emph{induced} subgraph). The third 
rule allows to select nodes of the partition as candidates for a clique. 
The final constraint requires that it is not possible that two nodes $X$ and 
$Y$ are in a clique and there is no edge in the subgraph $G_f$ from $X$ to $Y$.
Finally, the program $C$ is defined as follows. 

\begin{center}
$C=\left\{
\begin{array}{cl}
\leftarrow& \#\mbox{count}\{X:\inClique(X)\}<k
\end{array}\right\}$
\end{center}
The constraint forces the number of nodes in a clique to be greater or equal to $k$.

Intuitively, we check if for each answer set of $P_1$, that is for each total function $f$ from $I$ to $J$, there exists an answer set of $P_2$, that is a clique in the subgraph of $G$ induced by $f$, such that its cardinality is not less than $k$. 
If so, a quantified answer set of $\Pi$ exists. %Therefore, it holds t.

\begin{theorem}
Let $\mathcal{I}=\langle G, (A_{i,j})_{i\in I,j\in J}, k\rangle$ be an instance
of the {\sc Minmax Clique} problem. Then, 
\[
%\min_{f\in J^I}\;\max_{Q\subseteq N}\{|Q|: \mbox{$Q$ is a clique of $G_f$}\}\geq k
\min_{f\in J^I}\;\max\{|Q|: \mbox{$Q$ is a clique of $G_f$}\}\geq k
\]
if and only if the \QASP program $\Pi$, defined as above, has a quantified 
answer set.
\end{theorem}

\subsection{Pebbling Number}

Graph pebbling is a well-known mathematical game~\cite{Hurlbert:739273}. It
was first suggested as a tool for solving a particular problem in number 
theory~\cite{Chung:1989:PH:75533.75537}. The game consists of a graph with 
pebbles placed on (some of) its nodes. The goal is to place a pebble on a 
\emph{target} node by performing a sequence of \emph{pebbling} moves.
More formally, let $G=\langle N,E\rangle$ be a directed graph whose nodes 
may contain pebbles. A \textit{pebbling move} along an edge $(a,b)\in E$ 
requires that node $a$ contains at least two pebbles; the move removes 
two pebbles from $a$ and adds one pebble to $b$. The \textit{pebbling number}, 
denoted by $\pi(G)$, is the smallest number of pebbles such that for
every assignment of $k$ pebbles to nodes of $G$ and for every node $w\in N$
(the target), some sequence (possibly empty) of pebbling moves results in a 
pebble on $w$. The {\sc Pebbling number} problem asks whether $\pi(G)$ is 
less than or equal to $k$. This problem is $\Pi_2^p$-complete, and it remains 
so also when the target node is part of the 
input~\cite{Milans:2006:CGP:1146764.1237635}.
(For the latter version,we redefine $\pi(G)$ accordingly.)

To capture the definition of the {\sc Pebbling number} problem we construct
an \QASP program $\Pi = \forall^{st} P_1 \exists^{st} P_2:C$. Its program $P_1$ 
is defined as follows:
\begin{center}
$P_1=\left\{
\begin{array}{rcll}
edge(a,b) && \forall (a,b)\in E\\
node(a) && \forall a\in N\\
pebble(i) &&  \forall i=0,1,\ldots,k\\
1 \{ onNode(X,N) :  pebble(N) \} 1 &\leftarrow& node(X)\\ 
&\leftarrow& \#\mbox{sum}\{N,X : onNode(X,N)\}\neq k&\\
1 \{ target(X) :  node(X) \} 1 && 
\end{array}\right\}$
\end{center}
The first two sets of facts encode the input graph, and the third one the 
set of integers that can serve as the number of pebbles a node can have.
The first rule  of the program (line 4) selects, for each node $X$, the 
number $N$ of pebbles on $X$. The second rule (line 5) ensures the total 
number of pebbles on all 
nodes of $G$ is $k$. The last rule selects exactly one node as the target
allowing any node to be selected. Thus, answer sets of $P_1$ capture all 
possible ``input configurations'' for $G$, each configuration defined by 
a distribution of $k$ pebbles among nodes of $G$ and the target node.

The ASP program $P_2$ in $\Pi$ is defined as follows:
%
%\begin{center}
%$P_2=\left\{
%\begin{array}{rcl}
%step(i) && \forall i=0,1,\ldots,k-1;\\
%onNode(X,N,0) &\leftarrow& onNode(X,N);\\
%1\{ moveOn(X,S) : node(X) \}1 &\leftarrow& step(S);\\
%&\leftarrow& moveOn(X,S),\ onNode(X,N,S),\ N<2; \\
%onNode(X,N+1,S+1) &\leftarrow& onNode(X,N,S), \ onNode(Y,M,S), \\ 
%&& edge(X,Y), \ M \geq 2,\ moveOn(Y,S);\\
%onNode(Y,M-2,S+1) &\leftarrow& onNode(X,N,S), \ onNode(Y,M,S), \\ 
%&& edge(X,Y), \ M \geq 2,\ moveOn(Y,S)\\
%\end{array}\right\}$
%\end{center}
%
\begin{center}
$P_2=\left\{
\begin{array}{rcl}
\step(i) && \forall i=0,1,\ldots,k-1\\
1 \{ \estep(S): \step(S)\} 1\\
\onNode(X,N,0) &\leftarrow& \onNode(X,N)\\
1\{ \move(X,Y,S) : \edge(X,Y) \}1 &\leftarrow& \step(S), \estep(T), 1\leq S,\ S \leq T\\
&\leftarrow& \move(X,Y,S),\ \onNode(X,N,S),\ N<2 \\ 
\affected(X,S) &\leftarrow& \move(X,Y,S)\\ 
\affected(Y,S) &\leftarrow& \move(X,Y,S)\\ 
\onNode(X,N-2,S) &\leftarrow& \onNode(X,N,S-1), \move(X,Y,S)\\
\onNode(Y,M+1,S) &\leftarrow& \onNode(Y,M,S-1), \move(X,Y,S)\\
\onNode(X,N,S) &\leftarrow& \onNode(X,N,S-1), \textit{not}\ \affected(X,S)\\
\end{array}\right\}$
\end{center}
The first set of facts (line 1) encodes all integers $i$ that can serve as
the number of pebbling moves. Since each pebbling move removes one pebble, 
any successful sequence of pebbling moves has length at most $k-1$. 
Consequently, we may (and do) restrict these integers to $0,1,\ldots,k-1$. 
The first rule of $P_2$ (line 2) selects a single integer to represent the 
number of pebbling moves. The second rule of $P_2$ (the next line) defines 
the initial state of the graph (before any pebbling moves). It is given by 
the initial distribution of pebbles obtained from an answer set of the 
program $P_1$ (we overload the notation here; the predicate $\onNode$ 
defining the intial configuration in $P_1$ is binary, while the predicate
$\onNode$ defined in $P_2$ is ternary; it has an additional argument to
represent the step). The third rule selects an edge for the pebbling move step 
$S=1,2,\ldots,T$, where $T$ is the end step (defined via $\estep$). \textcolor{black}{The 
constraint that follows imposes the pebbling move precondition:} there must 
be at least two pebbles on the node where the pebbling move originates. The 
next two rules define the two nodes affected by
the move. The last three rules define the state of the graph after the 
pebbling move in step $S$ (applied to the graph after $S-1$ pebbling moves). 
The first two of these three rules describe how the number of pebbles change 
on the nodes that are involved in the move. The last rule is the inertia rule 
that keeps the number of pebbles unchanged on all nodes unaffected by the move. 
Informally, answer sets of $P_2$ correspond to all valid sequences of 
pebbling moves that do not eliminate all pebbles and start in the initial 
state of the graph, together with the corresponding sequence of states of 
the graph.

Finally, the program $C$ in $\Pi$ is defined as follows. 
\begin{center}
$C=\left\{
\begin{array}{rcl}
%target(a) && \forall \mbox{ target node }a;\\
ok(W) &\leftarrow& onNode(W,N,S), \ target(W), \estep(T)\ N>0\\
&\leftarrow& target(W),\ not \ ok(W)
\end{array}\right\}$
\end{center}
%The first set of facts encode the set of target nodes.
First rule defines $ok(W)$ to hold whenever $W$ is a target node 
and there is a pebble on it after the last pebbling move $T$. The 
constraint ensures no answer set if $ok(W)$ has not been inferred.

Intuitively then, $\Pi$ is coherent precisely when for each assignment of $k$ 
pebbles to nodes of a given graph and for every choice of a target node (that
is, for every answer set $M_1$ of $P_1$) there is a sequence of pebbling 
moves of length at most $k-1$ (that is, there is an answer set $M_2$ for 
\textcolor{black}{%
$P_2 \cup fix_{P_1}(M_1)=P_1'$) such that the target node has a pebble on it (that is,
$C\cup fix_{P_1'}(M_2)$ has an answer set).
}

\begin{theorem}
Let $\mathcal{I}=\langle G, k\rangle$ be an instance of the Pebbling Number Problem.
Then,  $\pi(G)\leq k$ if and only if the \QASP program 
$\Pi$, defined as above, is coherent. 
%admits a quantified answer set.
\end{theorem}

\subsection{Vapnik-Chervonenkis Dimension}

The \emph{Vapnik-Chervonenkis dimension} (VC dimension) is a fundamental 
concept in machine learning theory~\cite{Vapnik2015}. The VC dimension is 
a measure of the capacity of a space of functions that can be learned by 
a statistical classification algorithm~\cite{DBLP:journals/jacm/BlumerEHW89}. 
In particular, it is the cardinality of the largest set of points that the 
algorithm can shatter.
In statistical learning theory, the VC dimension can predict probabilistic 
upper bounds on the test error of a classification 
model~\cite{DBLP:books/daglib/0097035}. Further applications include finite 
automata, complexity theory, computability theory, and computational 
geometry. %~\cite{DBLP:journals/iandc/BeigelKS95a}.%,DBLP:journals/iandc/BeigelKS95b}.

Here, we focus on the so-called \emph{discrete} VC dimension problem, where the 
considered universe is finite. The problem concerns families of 
subsets that are represented by Boolean circuits. However, we assume that 
the representation is given by a logic program capturing the corresponding 
formula. Specifically, we assume that a program $P_\mathcal{C}$ representing 
a family $\mathcal{C}$ of subsets of $U$ contains a unary predicate $true$, 
and that extensions of the predicate $true$ in answer sets of $P_\mathcal{C}$ 
are precisely the elements of $\mathcal{C}$. Constructing a program 
$P_\mathcal{C}$ from a Boolean circuit representing $\mathcal{C}$ is a matter 
of routine and can be accomplished in linear time.
Let $k$ be an integer, $U$
%=\{x_1, \ldots, x_t\}$ 
a finite set, and $\mathcal{C}=\{S_1,\ldots,S_n\} \subseteq 2^U$ a collection 
of subsets of $U$ represented by a program $P_\mathcal{C}$.  
The {\sc VC Dimension} problem asks whether there is a subset $X$ of $U$ of 
size at least $k$, such that for each subset $S$ of $X$, there exists $S_i$ 
such that $S=S_i\cap X$. The VC dimension of $\mathcal{C}$ is defined as
maximum size of such a set $X$ and is denoted by $VC(\mathcal{C})$. Hence, 
the {\sc VC Dimension} problem asks whether $VC(\mathcal{C})\geq k$. It is 
known that this problem (assuming a circuit or a program representation of 
$\mathcal{C}$) is 
$\Sigma_3^p$-complete~\cite{DBLP:journals/jcss/Schaefer99}.
We will show that the problem can be described by an \QASP program 
$\Pi = \exists^{st}P_1 \forall^{st} P_2 \exists^{st} P_3:C$. 
The ASP program $P_1$ is defined as follows:
\begin{center}
$P_1=\left\{
\begin{array}{rcll}
inU(x) && \forall x\in U\\
%c(i,x,1) && \forall i\in \{1,\ldots,n\}, \ x\in S_i;\\
%c(i,x,0) && \forall i\in \{1,\ldots,n\}, \ x\not\in S_i;\\
k \{ inX(X) :  inU(X) \} && 
\end{array}\right\}$
\end{center}
The set of facts in line 1 encodes the elements of the set $U$, while the 
choice rule in line 2 selects a subset $X$ of $U$ with at least $k$ elements.
It is clear that answer sets of $P_1$ are all subsetes of $U$ with at least 
$k$ elements. 

The ASP program $P_2$ consists of a single choice rule:
\begin{center}
$P_2=\left\{
\begin{array}{c}
 \{ inS(X) :  inX(X) \}  
\end{array}\right\}$
\end{center}
Thus, answer sets of $P_2$ are subsets of a set $X$ (determined by a selected
answer set of $P_1$).

For $P_3$ we simply take $P_{\mathcal{C}}$. Wlog, we may assume that
$P_\mathcal{C}$ shares no vocabulary elements with $P_1$ and $P_2$. 
Thus, for every possible ``input'' from $P_1$ 
and $P_2$, answer sets of $P_3'$, that is, $P_3$ extended with the input 
from $P_1$ and $P_2$, determine elements of $\mathcal{C}$ via extensions 
of the predicate $true$.

%We now define the program $P_3$ in $\Pi$. In the definition, we use the 
%following notation: for each literal $l$ occurring in the formula $\varphi$, 
%we write $\sigma(l)$ to stand for $true(x)$, if $l = x$, and for 
%$not \ true(x)$, if $l = \neg x$. 
%%
%\begin{center}
%$P_3=\left\{
%\begin{array}{rcll}
%atom(x_i) &&& \forall i=1,\ldots,t;\\
%\{ true(X) : atom(X) \}; &&\\
%conj_j &\leftarrow& \sigma(l^h_j) & \forall j=1,\ldots,m, \ h=1,2,3;\\
%sat &\leftarrow& conj_1,\ldots,conj_m;\\
%&\leftarrow& not \ sat
%\end{array}\right\}$
%\end{center}
%This program encodes a given formula $\varphi$ representing an input family 
%$\mathcal{C}$ and its models. Indeed, the first set of facts describes all 
%atoms occurring in $\varphi$. The subsequent choice rule selects a subset of 
%this set as a putative model of $\varphi$ by making some atoms of $\varphi$ true.
%The next three rules impose constraints that ensure the selected subset
%is a model if $phi$: a conjunct of $\varphi$ is satisfied whenever a literal 
%is evaluated as true; $\varphi$ is satisfiable precisely when all conjuncts 
%are satisfied; the selected candidate is a model if it allows us to infer
%$sat$. Thus, stable models of $P_3$ are subsets of the set of atoms in 
%$\varphi$ that are models of $\varphi$.
 
Finally, the program $C$ is defined as follows (understanding $true$
as defined above):
\begin{center}
$C=\left\{
\begin{array}{rcl}
 inIntersection(X) &\leftarrow& true(X), \ inX(X)\\ 
 &\leftarrow& inIntersection(X),\  not \ inS(X)\\ 
 &\leftarrow& not \ inIntersection(X),\  inS(X)\\ 
\end{array}\right\}$
\end{center}
The first rule collects into predicate $inIntersection$, 
the intersection of the selected set $S_i$ from $\mathcal{C}$ (represented 
by an answer set of $P'_3$ by means of the predicate $true$) and $X$,
a subset of $U$ selected via an answer set of $P_1$. The two constraints
force this intersection to coincide with the subset $S$ of $X$ (an answer 
set of $P_2$ extended with a selected answer set of $P_1$ as input 
representing $X$). 
%$inIntersection$ and $inS$ to be instantiated by the same set of terms. 
%Indeed, the first constraint requires that it is not possible that $X$ 
%is $inIntersection$ and $X$ is not in $inS$, while the second one requires 
%that it is not possible that $X$ is not in $inIntersection$ and $X$ is 
%in $inS$.

Intuitively, the program $\Pi$ is coherent when there 
exists an answer set $M_1$ of $P_1$ 
(that is, a subset $X$ of $U$ of size at least $k$) such that
for each answer set $M_2$ of $P_2'=P_2\cup fix_{P_1}(M_1)$
(that is, for each subset $S$ of $X$),
there exists an answer set $M_3$ of $P_3'=P_3\cup fix_{P_2'}(M_2)$ 
(that is, an element $S_i$ of $\mathcal{C}$),
such that $C \cup fix_{P_3'}(M_3)$ is coherent (that is, $S_i\cap X$ is equal to $S$).

\begin{theorem}
Let $\mathcal{I}=\langle U, \mathcal{C},k\rangle$ be an instance of the VC dimension problem.
Then,  $VC(\mathcal{C})\geq k$ if and only if the \QASP program 
$\Pi$ defined as above has a quantified answer set.
\end{theorem}

%%%%%%%%%%%%%%%%%%%%%%%%%%%%%%%%%%%%%%%%%%%%%%%%%%%%%%%%
\section{Related Work and Discussion}

We now compare \QASP with related work discussing pros and cons of 
the various approaches.

\paragraph{\bf \QASP vs QBF.}
We first compare our proposal with Quantified Boolean Formulas (QBF)~\cite{DBLP:series/faia/2009-185}.
QBF is a natural extension of propositional formulas with quantifiers $\exists$
(existential) and $\forall$ (universal) operating on propositional variables. QFB was motivated by questions arising from computational complexity~\cite{DBLP:conf/stoc/StockmeyerM73}. The problem of checking the satisfiability of a propositional formula (SAT) is the canonical problem for the complexity class NP. 
The addition of quantifiers increases the complexity of satisfiability problem 
(QSAT) to PSPACE~\cite{DBLP:journals/tcs/Stockmeyer76}, and prefixes of $k$
alternating quantifiers yield problems that are complete for each complexity 
class of the Polynomial Hierarchy.  For this reason the satisfiability problem 
of QBF formulas with prefixes of alternating $k$ quantifiers ($k$-$QSAT$ 
becomes the canonical problem for the $k$-th level of the Polynomial Hierarchy).
More precisely, $k$-$QSAT$ restricted to prefixes of length $k$ starting with an existential 
(resp.  universal) quantifier is complete for $\Sigma_k^P$ (resp. $\Pi_k^P$).
\QASP and QBF share the same motivation and intuition, indeed \QASP extends ASP with quantifiers (as QBF extends SAT) to increase the modeling capabilities of the language beyond NP. 
As studied in Section~\ref{sec:complexity}, propositional \QASP and QBF have 
similar computational properties. In particular, the coherence problem for 
both is PSPACE-complete and an even tighter correspondence holds between  
propositional normal \QASP and QSAT. %stay on the same complexity class.
Nonetheless, there are important differences among the two languages, some 
inherited form the relation between SAT and ASP, and other concerning the 
semantics of quantifiers.

First, \QASP supports variables, which gives a modeling advantage, and supports
rapid prototyping, program optimization and maintenance of problem solution.
Indeed, variables allow one to encode uniform compact representation of 
a problem over varying instances, while in QBF (as in SAT) each instance of 
a problem needs to be encoded in a specific formula by means of an encoding
procedure. 
Second, even if in general QBF and \QASP can solve the same computational 
problems, \QASP inherits from ASP the possibility of encoding \textit{inductive 
definitions}~\cite{DBLP:conf/kr/DeneckerV14}, which are useful in modeling properties such as 
reachability in graphs (inductive definitions require larger instances in SAT
and QBF that slow down modeling and solving).
Next, ASP supports modeling extensions such as aggregates, choice rules, strong negation,
and disjunction in rule heads that significantly simplify encodings used 
in SAT~\cite{DBLP:journals/cacm/BrewkaET11}. We have made extensive use of
inductive definitions and aggregates in our examples in Section~\ref{sec:modeling}.
Finally, we note that in QBF quantifiers range over variable assignments, 
whereas in \QASP they quantify over the answer sets of each subprogram. 
This is yet another difference and a reason that ASP(Q) cannot be seen as a 
straightforward porting of the ideas behind QBF.

\paragraph{\bf \QASP vs ASP.} 
One of the distinguishing features of ASP is the capability of modeling 
problems in $\Sigma_2^P$. This is possible because of the additional 
expressive power provided by disjunctive rules. 
%\QASP restricted to programs with prefixes consisting of a single existential 
%quantifier coincides with ASP (see Theorem~\ref{th:extension}), one might 
%wonder whether the usage of two quantifiers can bring some advantage from the 
%perspective of modeling.
Modeling in $\Sigma_2^P$ problems with ASP is rather natural if one can use 
only \emph{positive} rules. For example, let us consider the \textit{strategic 
companies} problem~\cite{DBLP:journals/tkde/CadoliEG97}. In that problem, 
one has to compute a set of companies that cover the production of a set of 
goods also controlling other companies. A set of companies $S$ is said to be 
strategic if it: ($i$) covers the productions of all goods; ($ii$) is 
subset-minimal; and, ($iii$) every company $c$ controlled by at most three 
strategic companies is also strategic. In the setting in which each product 
is produced by at most two companies the problem is $\Sigma_2^P$-complete and 
can be modeled as follows~\cite{DBLP:journals/tocl/LeonePFEGPS06}: 

\begin{center}
$\begin{array}{rcl}
strat(Y) \vee strat(X) &\leftarrow& prod\_by(P,X,Y)\\ 
strat(W) &\leftarrow& contr\_by(W,X,Y,Z), strat(X), strat(Y), strat(Z)\\ 
\end{array}$
\end{center}
The first rule models condition $(i)$, the second rule models condition 
($iii$), and the minimality of answer sets ensures ($ii$). It is clear 
that this encoding of the problem can be directly translated to a 
single-quantifier disjunctive \QASP. 

%Note that this behavior can be obtained because all the conditions are modeled 
%by a positive program. Under this condition modeling in ASP is natural, 
%but 
When problem constraints to be modeled involve negation, ASP modeling becomes 
less intuitive. In particular one has to resort to an encoding technique called 
\emph{saturation}~\cite{DBLP:journals/amai/EiterG95}. It allows one to simulate 
a co-NP check in the program reduct. Saturation is at the basis of the 
celebrated encoding of 2-QBF by Eiter and Gottlob~\citeyear{DBLP:journals/amai/EiterG95} used to 
prove the complexity of checking existence of answer sets in presence of 
disjunction in rule heads. Given a 2-QBF formula $\Phi=\exists X\forall Y G$, 
where $G=D_1 \lor\ldots \lor D_h$ is a DNF, and
$D_i= L_{i,1} \land\ldots \land L_{i,k_i}$ and $L_{i,j}$ are literals 
over $X\cup Y$, we encode $\Phi$ in an ASP program as follows.
First introduce a fresh atom $sat$ modeling satisfiability, and
a fresh atom $nz$ for every atom $z\in X\cup Y$; and  
set $\sigma(z)=z$ and $\sigma(\neg z)=nz$ for every $z\in X\cup Y$.
Then write the program 
$P_\Phi = \{ z\vee nz | \forall z\in X\cup Y \} \cup \{ y \leftarrow sat | \forall y\in Y \}$$ \cup \{ ny \leftarrow sat | \forall y\in Y \} \cup \{ sat \leftarrow \sigma(L_{i,1}),\ldots,\sigma(L_{i,k_i}) | i=1,\ldots,m \} \cup \{ sat \leftarrow not\ sat \}$.

%\begin{center}
%$\begin{array}{rcll}
%z\vee nz &&&\mbox{for each } z\in X\cup Y\\
%y &\leftarrow &sat & \mbox{for each } y\in Y\\
%ny &\leftarrow &sat & \mbox{for each } y\in Y\\
%sat &\leftarrow& \sigma(L_{i,1}),\ldots,\sigma(L_{i,k_i})& \mbox{for each } i=1,\ldots,m\\
%sat &\leftarrow& not\ sat&
%\end{array}$
%\end{center}

%\begin{center}
%$\begin{array}{rcll}
%z\vee nz &&&\mbox{for each } z\in X\cup Y\\
%y &\leftarrow &sat & \mbox{for each } y\in Y\\
%ny &\leftarrow &sat & \mbox{for each } y\in Y\\
%sat &\leftarrow& \sigma(L_{i,1}),\ldots,\sigma(L_{i,k_i})& \mbox{for each } i=1,\ldots,m\\
%sat &\leftarrow& not\ sat&
%\end{array}$
%\end{center}
Here the atoms corresponding to universally quantified variables $Y$ are 
``saturated'' (i.e., they are forced to be true in any answer set), and 
since the last rule is always removed while computing the reduct, $sat$ 
must be derived for all assignments of truth values to $Y$ to have an answer 
set. This trick ensures that $\Phi$ is satisfiable if and only if $P_\Phi$ 
has an answers set. Again, one could reformulate the program above into
a \textit{disjunctive} program with a single quantifier. % (with disjunctions in the rule heads).
However, using saturation in modeling is considered difficult. 
\QASP offers an alternative and more intuitive approach, 
%Clearly, problems in $Sigma_2^P$ can be modeled by a quantified 
%program with a single existential quantifier over a disjunctive logic program
%(see Theorem~\ref{th:extension}). This approach suffers from the same need for
%saturation as in the disjunctive ASP. However, there is another way to use 
%\QASP for the task, specifically,
It uses \emph{normal} quantified programs with \emph{two} quantifiers that 
also capture $\Sigma_2^P$ (see Theorem~\ref{th:norm}). Indeed, 
let us consider a normal quantified program $\Pi_\Phi=\exists^{st} P_1 \forall^{st} P_2: C$ where 
$$P_1 = \{ \{x_1,\dots,x_n\} \},\quad P_2 = \{ \{y_1,\dots,y_m\} \},$$\vspace*{-0.6cm}
$$C = \{ sat \leftarrow \sigma(L_{i,1}),\ldots,\sigma(L_{i,k_i}) \mid \forall i=1,\ldots,m\} \cup \{ \leftarrow \textcolor{black}{not}\ sat \}.$$
Here, a satisfiability of an existential 2-QBF is encoded directly.
%without the need for saturation or other modeling devices.
Indeed $P_1$ guesses an 
assignment to $X$ s.t. for all assignments to $Y$ generated by $P_2$, 
$sat$ must be derived by satisfying at least one conjunct in $\varphi$, i.e., 
$\Pi_\Phi$ is satisfiable iff $\Phi$ is. 
This discussion suggests that \QASP improves on ASP modeling capabilities. 
It keeps the advantages of ASP in modeling concisely $\Sigma_2^P$
problems with positive programs, as for strategic companies, but also allows 
us to model other problems without resorting to difficult to use encoding techniques.

\paragraph{\bf \QASP vs Stable-Unstable.} 
To handle problems beyond NP, Bogaerts et 
al.~\citeyear{DBLP:journals/tplp/BogaertsJT16} proposed an extension of ASP
inspired by an internal working principle of ASP 
solvers~\cite{DBLP:conf/ijcai/GebserLMPRS18}. 
Usually, in ASP solvers designed for problems in $\Sigma_2^P$ 
one procedure generates model candidates and another one, acting as an oracle, 
tests minimality of the candidates produced by the first procedure. 
It does so by verifying that a certain subprogram (in some cases, a SAT 
formula) has no stable models (is not satisfiable). 
Following this principle, Bogaerts et al. \citeyear{DBLP:conf/ijcai/GebserLMPRS18} 
introduced \textit{combined logic programs}, in which two normal logic programs play a role 
analogous to the one of the two procedures of ASP solvers mentioned above.
A combined logic program is a pair $\Pi=(P_g , P_t)$ of normal logic programs. 
Its semantics is given by parameterized stable models~\cite{DBLP:conf/ecai/OikarinenJ06,DBLP:conf/iclp/DeneckerLTV12}; a \textit{stable-unstable model} of a combined program $\Pi$ 
is a parameterized stable model of $P_g$, say  $I$, such that 
no parameterized stable model of $P_t$ exists that coincides with $I$ in the 
intersection of the signatures of the two programs.

Comparing \QASP programs with combined programs, we first 
note that combined programs involve the concept of parameters. In applications, 
the parameters of the generator program are used to represent problem 
instances (are ``extensional''). This use of parameters is quite natural to 
ASP programmers and does not pose a conceptual difficulty. It is also used 
implicitly in \QASP (stable models from each quantifier are passed on as 
``input'' parameters to the next one).\footnote{We could also distinguish
extensional predicates to specify ``parameters,'' that is, input instances, 
That would allow us to keep instance specification separate from the program.
We decided not to do so here to simplify our presentation.}
However, the stable-unstable approach applies the notion of a parameterized 
stable model also in the checking phase using ``negation,'' that is, 
referring to non-existence of a certain parameterized stable model. This,
arguably, makes the formalism much less direct than \QASP. It is especially 
clear when we move beyond the second level of the PH and 
the non-existence conditions become nested (incidentally, the stable-unstable 
paper contains no examples of modeling such problems).

If we factor out the issue of parameters, and limit ourselves to problems
in $\Sigma_2^P$, combined programs and \QASP are closely related. 
Indeed, in \QASP one has direct means to model ``testing'' 
conditions of the form ``for all stable models (answer sets) of some  program, 
a certain property holds.'' In contrast, combined programs provide direct means 
to model ``testing'' conditions of the form ``there exists \emph{no} stable 
model of some program such that a certain property holds.'' 
%, and one can switch from one to the other with very minor effort.
Switching between \QASP and combined programs amounts then to simulating 
conditions of one form with conditions of the other and \emph{vice versa}
(effectively, negating constraints in a program).
%conditions ``for all $M$, $F$'' by ``there is no $M$, such that $\neg F$'' and
%\emph{vice versa}. 
Such simulations are easy to design with the use of a small number of auxiliary variables (often one such new variable suffices).
Consequently, both formalisms are on par for modeling problems that are 
complete for $\Sigma_2^P$. However, for problem in $\Pi_2^P$, the difference 
between \QASP and combined programs becomes evident. 
As an example, let us consider a 2-QBF formula $\Psi=\forall X\exists Y \psi$, 
where $\psi$ is a 3-CNF formula. This problem can be naturally represented
in \QASP by using the encoding employed in the proof of Theorem~\ref{th:pspace}.
However once we try to encode it using a combined logic program (for well-known complexity reasons) we have either to adopt an exponential encoding, something 
analogous to quantifier expansion in QBF, or we have to use an additional 
nesting of programs (i.e., we are have to push the entire computation in the 
oracle). In both cases, the modeling would not result in a solution as natural 
and direct as the one provided by \QASP. The reason is  that combined programs 
(as well as their generalizations beyond the second level) represent existential
statements. Hence, they model \textit{complements} of $\Pi_2^P$ problems and 
not the problems themselves. In contrast, \QASP can be used for such problems
in a direct way providing representations closely following original problem descriptions (our examples illustrate this).  

A related aspect concerns modeling itself, the process of 
mapping natural language 
specifications to formal expressions, which surfaces when one considers 
problems that require more than one quantifier alternation. It is important 
to note that combined logic programs were extended to deal with problems from 
any level of the PH in~\cite{DBLP:journals/tplp/BogaertsJT16}
by resorting to a recursive definition. This definition forces the programmer 
to think in terms of ``nested oracles'', instead of translating problem 
description directly into a formal expression. Whereas for problems at the 
second level of the polynomial hierarchy it roughly corresponds to searching 
for a counterexample, for problems at higher levels, the recursion and the 
negation (needed because of the absence of direct means to represent universal 
statements), makes it harder to maintain the connection between problem 
description and oracles forming nested combined programs. 
%During modeling, this requires pushing negation in the oracle several 
%times, or adopting encoding solutions to complement the output of the 
%oracle.
In contrast, the interface between natural language problem description 
and \QASP programs is transparent (in the same way as it is for QBF), as 
it is explicitly supported by the quantifiers, which may be existential or 
universal, as needed. In particular, the difficulty of modeling problems 
in $\Pi_2^P$, noted above, appears in the general setting of 
problems in $\Pi_k^P$, for $k\geq 2$: the stable-unstable formalism is 
not designed to directly express universal statements that characterize
problems in $\Pi_k^P$.

\textcolor{black}{The discussion above compares at an intutive informal level 
the modeling freatures of the two formalisms. It also suggests how the two are
formally related. In the statement specifying the relation, the \emph{depth} 
of the basic combined program is defined as 2. Each next level of nesting 
increments the depth by 1. } 

\textcolor{black}{
\begin{theorem}
(i) There is a polynomial-time reduction that assignes to every propositional 
nested combined program $\Pi$ of depth $n$, a normal existential \QASP program 
$\Pi_q$ with $n\geq 2$ quantifiers such that answer sets of $\Pi$ and $\Pi_q$, 
correspond to each other.\\
(ii) There is a polynomial-time reduction that assignes to every propositional
normal existential \QASP program $\Pi$ with $n\geq 2$ quantifiers in the 
prefix, a propositional nested combined program $\Pi_c$ of depth $n$ such that
answer sets of $\Pi$ and $\Pi_c$ correspond to each other.
\end{theorem}
}

\textcolor{black}{
Thus, at the level of expressive power, combined programs of depth $n$ and
existential $\QASP$ programs with $n$ quantifiers are formally equivalent, 
even if from the modeling point of view, as we argued, \QASP programs seem 
to have an advantage. However, unless the polynomial hierarchy collapses, 
no reduction from universal $\QASP$ programs with $n$ quantifiers to 
combined nested programs of depth $n$ is possible.
The following proposition specifies this property for the particular case of the validity of 2-QBFs, which we discussed above.
}
\textcolor{black}{
\begin{proposition}
Unless the polynomial hierarchy collapses, there exists no polynomial reduction that encodes formulas $\Psi=\forall X\exists Y \psi$, where $\psi$ is a 3-CNF formula, as a combined program $P=(P_1,P_2)$, where $P_1$ and $P_2$ are normal logic programs, such that $\Psi$ is valid iff $P$ admits stable unstable models.
\end{proposition}
}\textcolor{black}{
A trivial consequence of Theorem~\ref{th:pspace} is that this limitation is absent from \QASP. 
}

Finally, we note that combined programs under stable-unstable semantics have 
been implemented in a proof of concept prototype~\cite{DBLP:journals/tplp/BogaertsJT16} that can only handle problems at the second level of the polynomial 
hierarchy. 
A similar prototype implementation for \QASP (programs with at most two quantifiers) is possible, too.
%A similar prototype implementation for \QASP restricted to problems 
%at the second level of the polynomial hierarchy is possible, too. 
However,  devising efficient implementations for either formalism in their full 
generality remains a non-trivial open research problem.

\paragraph{\bf Further related work.} 
The problem of modeling in a natural way $\Sigma_2^P$ problems with ASP was
also addressed by Eiter and Polleres~\citeyear{DBLP:journals/tplp/EiterP06}. 
They model problems combining ``guess'' program $P_{solve}$ and ``check'' 
program $P_{check}$, which are transformed into a single disjunctive ASP 
program such that its answer sets encode the solutions of the original 
problem by means of a polynomial-time transformation. The programs $P_{solve}$ 
and  $P_{check}$ must be HCF and propositional, thus limiting this 
approach to the modeling capabilities of propositional ASP. An idea analogous
to that developed by Eiter and Polleres~\citeyear{DBLP:journals/tplp/EiterP06} 
was also proposed by Redl~\citeyear{DBLP:conf/lpnmr/Redl17a}. Redl's proposal
appears to be conceptually simpler than the earlier one because of the use of 
conditional literals but suffers from the same limitations.
A general technique to reuse existing ASP systems to evaluate problems of 
higher complexity (such as  various forms of qualitative preferences among 
answer sets) was proposed by Gebser et al.~\citeyear{DBLP:journals/tplp/GebserKS11}. 
The idea there was to use a meta program encoding the saturation technique
which, in this way, became transparent to the user. As in the approach by Eiter
and Polleres~\citeyear{DBLP:journals/tplp/EiterP06}, the resulting program is a 
plain ASP program which can be evaluated by a standard ASP system. 
Thus, the approach of Gebser et al.~\citeyear{DBLP:journals/tplp/GebserKS11} 
cannot be used to model problems beyond the second level of the polynomial hierarchy.
Another solution that allows for reasoning within a program over the answer sets of another program, and thus encode reasoning tasks beyond NP, is provided by manifold programs~\cite{DBLP:conf/birthday/FaberW11,Faber200934}. 
In manifold programs the calling and the called program are encoded into a single program using weak constrains. The answer sets of the called program are thus represented within each answer set of the calling program. Also this approach is limited to the second level of the polynomial hierarchy, and might generate large specifications.

HEX-programs are an extension of ASP with external sources such as description 
logic ontologies and Web resources~\cite{DBLP:journals/ai/EiterILST08}. In 
HEX-programs external atoms can exchange information from the logic program 
to eternal theories in terms of predicate extensions and constants.
\textcolor{black}{
Redl~\citeyear{DBLP:conf/lpnmr/Redl17a} studied a way to avoid saturation for modeling $\Sigma_2^P$ problems with HEX-programs.} 
In particular, the author proposes the modeling technique of query answering 
over subprograms. While encoding a problem on the second level of the 
polynomial hierarchy, one has to provide two components. A first program 
$P_{guess}$ modeling the NP part, and a second one $P_{check}$ modeling the 
co-NP check. The first program, $P_{guess}$, is a HEX program that can query 
on the answer sets of the normal ordinary ASP program $P_{check}$ using 
specific external atoms. This modeling approach avoids saturation without 
introducing quantifiers, but this nice modeling behavior is limited to 
$\Sigma_2^P$ problems. Indeed, the focus of query answering over subprograms 
is on overcoming saturation and not on reaching high 
expressibility~\cite{DBLP:conf/lpnmr/Redl17a}.
A recent proposal of an extension of propositional ASP to model planning problems was described in~\cite{Romero2017}. 
The main difference with \QASP is on the nature of quantifiers allowed in the two specifications.
Indeed, the proposal of ~\cite{Romero2017}, mimicking 2QBF, allows quantifiers over propositional atoms, whereas in \QASP quantifiers are over answer sets. 
%Although an unrestricted number of quantifiers could be supported as in \QASP, this proposal shares pros and contra with QBF.
%an complete extension to non-propositional programs is not straightforward since quantifiers range over ground atoms.

As a final mention, we observe that the idea of extending the base language 
with quantifiers has been applied also in the neighboring area of Constraint 
Satisfaction Problems (CSP)~\cite{DBLP:reference/fai/RossiBW06}, obtaining 
Quantified CSP (QCSP)~\cite{DBLP:conf/cp/BordeauxM02}. 

%%%%%%%%%%%%%%%%%%%%%%%%%%%%%%%%%%%%%%%%%%%%%%%%%%%%%%%%
\section{Conclusions}
In this paper we approached the modeling of problems beyond NP with ASP programs.
Inspired by the way QBFs extend SAT formulas, we have introduced \QASP, which extends ASP via quantifiers over stable models of programs. 
We have studied the computational properties of the language, provided a number of examples to demonstrate its modeling capabilities, and compared alternative approaches to the same problem.
The analysis provided in the paper suggests that \QASP is able to model uniformly problems in the Polynomial Hierarchy in the same compact and elegant way as ASP models problems in NP. 

The definition of \QASP allows for disjunctive programs, thus all the features 
of the basic language are retained. 
However, by limiting to normal (or HCF) programs (extended with aggregates and 
other useful modeling constructs) in \QASP, one can take advantage of the 
classic generate-define-test modular programming methodology and other 
modeling techniques developed for these best understood classes of programs
to model any problem in the Polynomial Hierarchy.
Indeed, the presence of quantifiers allows one to model complex properties in a direct way, without the need of recasting them in terms of checking the minimality of a model, e.g., using saturation.
The examples provided in the paper, indeed, employ normal programs, and the solutions follow directly from the definition in natural language of the problem at hand. 

The key task for the future is to implement \QASP. In this respect many possible solutions are possible, from encoding \QASP in QBF and resorting to QBF solvers, to evolving ASP solvers to handle quantifiers over stable models.

\section*{Acknowledgements}
The work of the third author has been partially supported by the NSF grant IIS-1707371.
This work has been partially supported by MIUR under PRIN 2017 project n. $2017M9C25L\_001$ (CUP $H24I17000080001$). 

\bibliographystyle{acmtrans}
\bibliography{biblio}

\newpage

\appendix

\setcounter{theorem}{0}

\label{lastpage}
\end{document}